\documentclass[letterpaper]{article} 
\usepackage{aaai24}  
\usepackage{times}  
\usepackage{helvet}  
\usepackage{courier}  
\usepackage[hyphens]{url}  
\usepackage{graphicx} 
\usepackage{natbib}  
\usepackage{caption} 
\usepackage{algorithm}
\usepackage{algorithmic}

\usepackage{newfloat}
\usepackage{listings}
\DeclareCaptionStyle{ruled}{labelfont=normalfont,labelsep=colon,strut=off} 
\floatstyle{ruled}
\newfloat{listing}{tb}{lst}{}
\floatname{listing}{Listing}
\usepackage{multirow}
\usepackage[utf8]{inputenc} 
\usepackage[T1]{fontenc}    
\usepackage{url}            
\usepackage{booktabs}       
\usepackage{amsfonts}       
\usepackage{nicefrac}       
\usepackage{microtype}      
\usepackage{xcolor}         
\usepackage{amsmath}
\usepackage{subfigure}
\usepackage{mathtools}

\usepackage{amsthm}
\newtheorem{theorem}{Theorem}
\newtheorem{prop}[theorem]{Proposition}
\newtheorem{cor}[theorem]{Corollary}
\newtheorem{lemma}[theorem]{Lemma}

\newtheorem{remark}[theorem]{Remark}

\newtheorem{assumption}[theorem]{Assumption}
\makeatletter

\newcommand{\Rmnum}[1]{\expandafter\@slowromancap\romannumeral #1@}
\makeatother

\DeclareMathOperator*{\arginf}{arg\,inf}
\DeclareMathOperator*{\argmin}{arg\,min}

\def\A{\mathcal{A}}
\def\X{\mathcal{X}}
\def\Y{\mathcal{Y}}
\def\G{\mathcal{G}}

\def\Dd{\mathcal{D}}
\def\L{\mathcal{L}}
\def\J{\mathcal{J}}
\def\S{\mathcal{S}}
\def\N{\mathbb{N}}

\def\R{\mathbb{R}}
\newcommand{\E}{\mathop{\mathbb{E}}\limits}

\title{Game-Theoretic Unlearnable Example Generator}
\author{
    Shuang Liu\textsuperscript{\rm 1, 2},
    Yihan Wang\textsuperscript{\rm 1, 2},
    Xiao-Shan Gao\textsuperscript{\rm 1, 2}\thanks{Corresponding author.}
}
\affiliations{
    \textsuperscript{\rm 1}Academy of Mathematics and Systems Science, Chinese Academy of Sciences\\
     \textsuperscript{\rm 2}University of Chinese Academy of Sciences\\

{liushuang2020@amss.ac.cn,  yihanwang@amss.ac.cn,
xgao@mmrc.iss.ac.cn}
}

\usepackage{bibentry}

\begin{document}

\maketitle

\begin{abstract}
Unlearnable example attacks are data poisoning attacks aiming to degrade the clean test accuracy of deep learning by adding imperceptible perturbations to the training samples, which can be formulated as a bi-level optimization problem. However, directly solving this optimization problem is intractable for deep neural networks.
In this paper, we investigate unlearnable example attacks from a game-theoretic perspective, by formulating the attack as a nonzero sum Stackelberg game. First, the existence of game equilibria is proved under the normal setting and the adversarial training setting. 
It is shown that the game equilibrium gives the most powerful poison attack in that the victim has the lowest test accuracy among all networks within the same hypothesis space,  when certain loss functions are used.
Second, we propose a novel attack method, called the Game Unlearnable Example (GUE), which has three main gradients. 
(1) The poisons are obtained by directly solving the equilibrium of the Stackelberg game with a first-order algorithm. 
(2) We employ an autoencoder-like generative network model as the poison attacker.
(3) A novel payoff function is introduced to evaluate the performance of the poison. Comprehensive experiments demonstrate that GUE can effectively poison the model in various scenarios. 
Furthermore, the GUE still works by using a relatively small percentage of the training data to train the generator, and the poison generator can generalize to unseen data well. Our implementation code can be found at https://github.com/hong-xian/gue.
\end{abstract}

\section{Introduction}
Deep learning has achieved remarkable success in various fields, including computer vision \cite{resnet}, natural language processing, and large language models \cite{brown2020gpt}, where the acquisition of a substantial amount of training data is typically necessary. 
However, in practical scenarios, there exists a potential issue of collecting unauthorized private data from the Internet to train these models. This raises significant privacy concerns and underscores the need to address how to protect personal data from exploitation.

To prevent unauthorized use of private data, several poisoning methods \cite{zhou2019confuse, huang2021unlearnable, adv_poison, yuan2021ntga, tao2021better, yu2022shortcut, shortcutgen, sandoval2022autoregressive} have been proposed. These methods involve adding imperceptible perturbations to the training samples, thereby poisoning the models and leading to poor performance on clean test data, thus making the training data unlearnable. However, it is generally agreed that the above methods are vulnerable to adversarial training. In response, novel approaches such as \cite{fu2022rem, 2023inf} have been proposed to generate robust unlearnable examples that can withstand adversarial training. Commonly referred to as "unlearnable examples attack" or "availability attack", these poison attacks have attracted significant attention. 

In previous works \cite{zhou2019confuse, tao2021better, yu2022shortcut}, unlearnable example attacks were usually formulated as a bi-level optimization problem. But directly solving the optimization problem is intractable, and \citet{zhou2019confuse} used an alternative update and \citet{yuan2021ntga} used Neural Tangent Kernels \cite{jacot2018NTK} to approximately optimize the bi-level objective. 
\citet{yu2022shortcut} designed poison perturbations that are easily learned by the model as ``shortcuts'' to prevent learning the information from the real data.
%

In this paper, we formulate the unlearnable example attack as a non-zero sum Stackelberg game, in which the attacker, as the leader, crafts poison perturbation on the training data with the aim of decreasing the test accuracy, while the classifier, as the follower, optimizes the network parameters on the poisoned training dataset. 
We prove the existence of Stackelberg equilibria in both the normal and adversarial training settings.
Furthermore, we propose a novel and effective approach, namely Game Unlearnable Example (GUE), in which we directly compute the Stackelberg game equilibrium using a first-order algorithm \cite{bome}
and select an appropriate payoff function for the poison attacker.
It is noteworthy that our GUE is not only applicable to standard natural training but can also extend to adversarial training. 

Existing unlearnable example attacks typically require modifying the entire training dataset.
%
However, in practical scenarios, the continuous influx of new clean data over time can render the unlearnable effect almost negligible. 
To mitigate this problem, when there is a continuous stream of new data, the user must go through the generation process for the whole dataset repeatedly, which is time-consuming. 
In order to overcome this problem, we use an autoencoder-like generator as the attacker, and show that a well-trained attacker can easily generalize to unseen data well. 
The generator can be used to perform a simple forward propagation on the incoming data to generate the corresponding poison perturbation. Figure \ref{fig-pipeline} illustrates the pipeline of the GUE attack.


\begin{figure}[t]
\centering
\includegraphics[width=0.9\columnwidth]{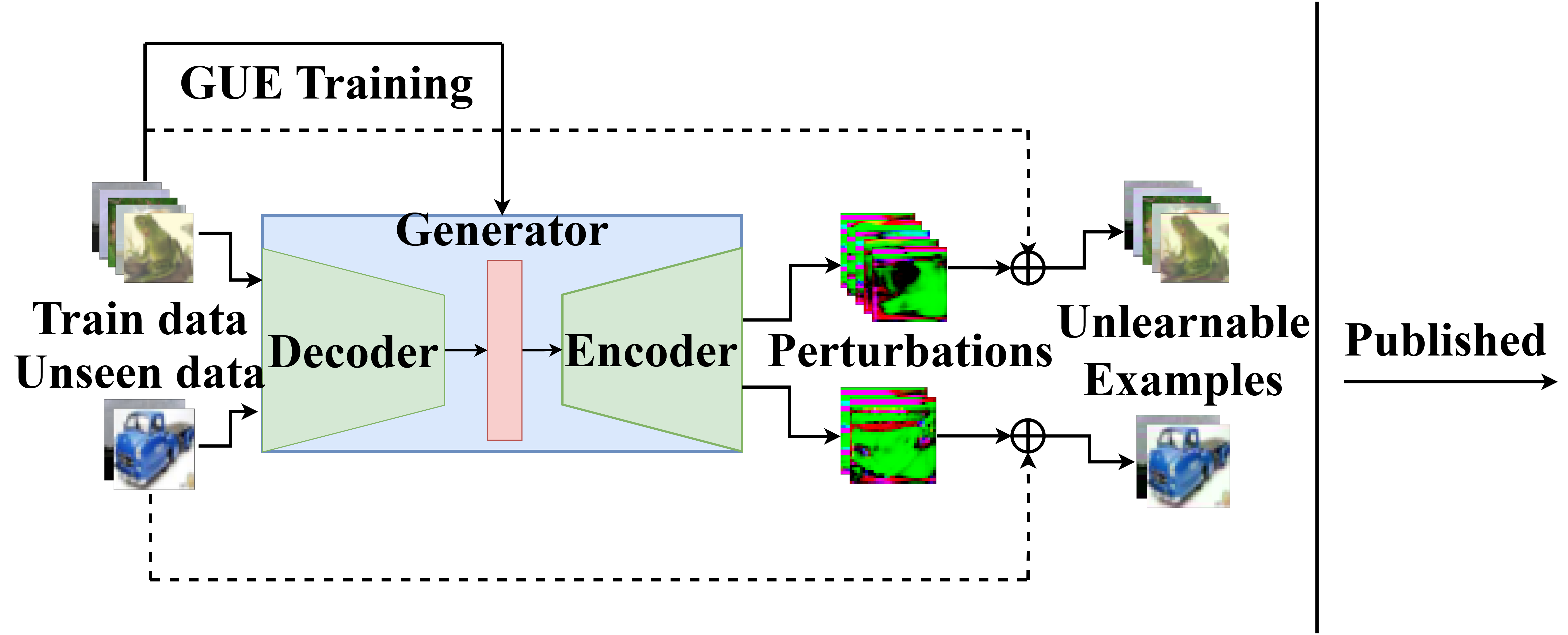} 
\caption{An illustration of GUE attack, where the trained generator can generalize to unseen data well.}
\label{fig-pipeline}
\end{figure}

In summary, our work has three main contributions:
\begin{itemize}
\item The unlearnable example attacks are formulated as a Stackelberg game and the existence of game equilibria under several useful settings are proved.
It is further shown that the game equilibrium gives the most powerful poison attack in that the victim has the lowest test accuracy among all networks within the same hypothesis space when certain loss functions are used.

\item We propose a new unlearnable example attack, called GUE, which, to the best of our knowledge, is the first poison method directly based on computing the game equilibria.
Furthermore, by using an autoencoder-like generator as the attacker, poison perturbations for new data can be generated with simple forward propagation.
\item Extensive experiments are used to demonstrate the effectiveness and generalizability of GUE in different scenarios. In particular, it is shown that GUE still works by using a relatively small percentage of the training data to train the generator.
    %
\end{itemize}

\section{Related Works}
Our GUE attack is intimately related to unlearnable example attacks, the Stackelberg game, and bi-level optimization problems. We first provide a brief overview of the developmental trajectory about unlearnable example attacks and explain their relationship with our approach.

Unlearnable example attacks are data poisoning methods with addictive bounded perturbations that prevent unauthorized model training. 
Error-minimizing noise (EM) \cite{huang2021unlearnable} generates an imperceptible perturbation by solving a Min-Min optimization.
Hypocritical perturbations (HYP) \cite{tao2021better} follows a similar idea but directly uses a pre-trained model rather than the above Min-Min optimization. 
Error maximizing noise (TAP) \cite{adv_poison} generates adversarial examples as poisoned data. 
NTGA \cite{yuan2021ntga} generates poison perturbations in an ensemble of neural networks modeled with neural tangent kernels. 
There exist other methods that do not involve optimization. LSP \cite{yu2022shortcut} synthesizes linearly separable perturbations as attacks, and AR \cite{sandoval2022autoregressive} introduces a generic perturbation that can be utilized in various datasets and architectures. 

However, the above methods cannot poison the adversarial trained model, and several novel methods were proposed to mitigate this problem.
REM \cite{fu2022rem} generates a stronger perturbation by solving a Min-Min-Max three-level optimization problem.
INF \cite{2023inf} generates a poison perturbation based on the induction of indiscriminate features between different classes to poison an adversarially trained model under a larger adversarial budget.

\citet{lu2022TGDA} formulated the traditional data poison attacks as a Stackelberg game and used total gradient descent to solve the game, while they considered adding a small portion of the poisoned data to the training set rather than modifying the training samples. And they did not consider the existence of game equilibrium.

The works most relevant to ours are DeepConfuse \cite{zhou2019confuse} and ShortcutGen\cite{shortcutgen}.
\citet{zhou2019confuse} formulated the unlearnable examples attack as a bi-level optimization problem, and optimized the bi-level problem by decoupling the alternating update procedure. However, this method cannot guarantee the convergence theoretically and requires multiple rounds of optimization, which is very time-consuming. 
An autoencoder-like generator was also used as the attacker in DeepConfuse, but the generalizability of the generator was not discussed.
\citet{shortcutgen} used a static randomly initialized discriminator to train a poison generator, in order to encourage the generator to learn spurious shortcuts in the Min-Min framework \cite{huang2021unlearnable}. 
It is based on their conjecture that a randomly initialized discriminator provides a noisy mapping between images and labels and thus lacks theoretical guarantee.

\section{Unlearnable Example Game}
\label{sec:game}
We formulate the unlearnable example attack
as a Stackelberg game, called \emph{unlearnable example game}, denoted as $\mathcal{G}$. 
Then, we prove the existence of its equilibrium in the general case, which is extended to various attacking scenarios.
\subsection{Unified Game Framework}
\paragraph{Settings.}
The game $\mathcal{G}$ has two participants: a poison attacker and a victim classifier. 
Let $\S$ be a finitely sampled clean training dataset from a data distribution $\Dd$ over $\X\times \Y$,
where $\X\subset\R^n$ is the dataset 
and $\Y=[K]=\{i\}_{i=1}^K$ is the label set for $K\in \N_+$.
The attacker modifies $x$ in each sample-label pair $(x,y) \in \S$ by an imperceptible perturbation $\A(x, y)$ such that
\begin{align*}
    ||\A(x, y)||_\infty \leq \epsilon.
\end{align*}
In our method proposed in Algorithm \ref{alg:algorithm},
$\A$ will be specified as an encoder-decoder generator. 
Here, it can be simply regarded as a map $A(x,y):\S\to \R^n$, where $n$ is the dimension of the data.
The victim classifier only has access to the poisoned training set on which a classifier $f$ with parameters $\theta$ will be trained.
The attacker's goal is to destroy the performance of $f_\theta$.
We denote the loss function used in classifier training by $\L_c$ and the loss function used in the evaluation of attack performance by $\L_a$.

\paragraph{Game model.}
In the presence of a poisoning $\A$ introduced by the game leader, the payoff function of the victim is the empirical risk in the training process:
\begin{align*}
    \J_c(\A,\theta)&\coloneqq \E_{(x,y)\sim \S} \big[\L_c(x+\A (x,y), y;\theta)\big].
\end{align*}
As the follower of the game, the victim should choose one of the \emph{best responses}:
\begin{align*}
    \theta^*\in \text{BR}(\A) \coloneqq \arginf_{\theta'}\J_c(\A, \theta').
\end{align*}
Due to the fact that neural network training usually reaches local optima, we consider the \emph{$\eta$-approximately best responses}:
\begin{align*}
\theta^*\in \text{BR}_{\eta}(\A) \coloneqq \{\theta | \J_c(\A, \theta) < \inf_{\theta'}\J_c(\A, \theta') + \eta\}.
\end{align*}
The payoff function of the attacker
is the negative value of population risk:
\begin{align}\label{equ:payoff-func-attacker}
    \J_a(\A,\theta) \coloneqq 
\sup_{\theta\in \text{BR}_{\eta}(\A)} \{- \E_{(x,y)\sim \Dd} \big[\L_a(x, y; \theta)\big]\}. 
\end{align}
The attacker should choose a \emph{Stackelberg strategy}
\begin{align}
\A^* \in\arginf_{\A}  \J_a(\A,\theta). \label{(2)}   
\end{align}
Though the Stackelberg strategy is not ensured unique, the Stackelberg cost is unique.
For a Stackelberg strategy $\A^*$, each $\theta^*\in \text{BR}_{\eta}(\A^*)$ is an approximate optimal strategy for the victim with respect to $\A^*$. 
We call the action profile $(\A^*, \theta^*)$ a {\em Stackelberg equilibrium} without distinction.
%
%

\subsection{Equilibrium Existence}\label{subsec:equilibrium-exist}
One of the primary concern for the game $\mathcal{G}$ is whether this game possesses an equilibrium.
We will prove the existence of the equilibrium, consequently providing a well-defined target for learning.
Next section will delve into the methods for solving this game.

Assume that the parameters $\theta$ of the victim classifier are restricted in a feasible strategy space $\Theta \subset \mathbb{R}^N$, where $N$ is the number of parameters.
The following common assumptions are required for the existence of a Stackelberg equilibrium. 
\begin{assumption}\label{ass:comp}
    The strategy space $\Theta$ is compact,
    for instance $\Theta=[-E,E]^N$ for $E\in\R_+$. 
\end{assumption}
\begin{assumption}\label{ass:l-lip}
    The loss functions $\L_a(x, y; \theta)$ and $ \L_c(x, y; \theta)$ are continuous in $x$ and $\theta$.
\end{assumption}

\begin{theorem}\label{thm:exist}
    Under Assumptions \ref{ass:comp} and \ref{ass:l-lip}, the unlearnable example game $\mathcal{G}$ has a Stackelberg equilibrium $(\A^*, \theta^*)$.
\end{theorem}
\begin{proof}[Proof sketch.]
Let $\Gamma$ be the feasible space for the perturbations on the training set, that is, $\A(\S)\subset \Gamma$.
With a slight abuse of notation, we denote $\A\in \Gamma$.

By Assumptions \ref{ass:comp} and \ref{ass:l-lip},   $\J_c(\A, \theta)$ is continuous, $\Theta$ and $\Gamma$ are compact. Then, for any $\A$, $\text{BR}_{\eta}(\A)$ is non-empty. Let $\J(\theta) \coloneqq - \E_{(x,y)\sim \Dd} \big[\L_a(x, y; \theta)\big]$. 
Thus the existence of Stackelberg equilibrium is equivalent to the existence of $\A^*\in \Gamma$ such that
\begin{align*}
\sup_{\theta\in \text{BR}_{\eta}(\A^*)}\J(\theta)= \inf_{\A} \sup_{\theta\in \text{BR}_{\eta}(\A)}\J(\theta).
\end{align*}
Since $\J(\theta)$ is continuous in $\Theta$, the marginal function 
\begin{align*}
    V(\A)\coloneqq \sup\limits_{\theta\in \text{BR}_{\eta}(\A)} \J(\theta)
\end{align*}
is lower semi-continuous. 
Then $V(\A)$ can attain the minimum as $\Gamma$ is compact, that is, there exists an $\A^*$ such that
\begin{align*}
    V(\A^*)= \inf_{\A\in \Gamma} V(\A)= \sup_{\theta\in \text{BR}_{\eta}(\A^*)}\J(\theta).
\end{align*}
Since $\text{BR}_{\eta}(\A)$ 
is not empty, there exists  a $\theta^*\in \text{BR}_{\eta}(\A^*)$, and the theorem is proved.
%
Details of the proof are given in subsection \ref{pf-th3}.
\end{proof} 

\subsection{Various Attacking Scenarios}
\label{subsec: various scenarios}
We have proved the existence of Stackelberg equilibrium with respect to general loss functions $\L_c$ and $\L_a$. 
Next, we will consider specific loss functions for different scenarios below.
As long as $\L_c$ and $\L_a$ satisfy Assumption \ref{ass:l-lip}, the specified unlearnable example game has an equilibrium.

\paragraph{Standard setting.}
The most common attack scenario for classification tasks leverages cross-entropy loss as both $\L_c$ and $\L_a$.
In this case, the attacker aims at reducing standard accuracy on the original data distribution while the victim trains classifier $f_\theta$ to improve standard accuracy on the poisoned training set.
We have the following corollary as a direct consequence of Theorem \ref{thm:exist}. 
\begin{cor}
     Let $\L_c=\L_a=\L_{ce}$, then the unlearnable example game $\G$ has a Stackelberg equilibrium.
\end{cor}

In practice, solving this game also requires taking into account the effectiveness of the algorithm.
Searching for a Stackelberg strategy of the attacker refers to minimizing its payoff function in Equation \eqref{equ:payoff-func-attacker} involving $\L_{ce}$.
However, the cross-entropy loss has no upper bound, and thus naturally lacks a criterion for convergence of such an optimization. 
Furthermore, it results in a gradient explosion, as will be discussed in the experimental section.

To tackle this issue, we instead leverage a surrogate loss function which is upper-bounded and is equivalent to cross-entropy loss for solving the game.

\paragraph{Surrogate loss.} 

We propose a novel loss function for the attacker to measure the poison performance:
\begin{equation}\label{eq:max}
    \L_{sur}(x, y; \theta) \coloneqq - \max_{k\in [K]\setminus\{y\}}\L_{ce}(x, k;\theta).
\end{equation}
%
Intuitively, maximizing $ \L_{sur}$ means minimizing the cross-entropy loss of $x$ with respect to all other labels instead of $y$, leading to the misclassification of the model trained on the poisoned training set. 
When maximizing $\L_{ce}$, it often forms a deep valley on the true label in the label confidence distribution.
On the other hand, when maximizing $\L_{sur}$, in addition to forming a deep valley on the true label, it also makes the confidence levels of other labels uniform.

The following properties of surrogate loss allow our method  
to effectively solve the game.
Moreover, an equilibrium still exists for surrogate loss.
\begin{prop}
Assume that the data have $K$ classes.
\begin{enumerate}
    \item $\L_{sur}$ is upper-bounded:
    \begin{align*}
        \L_{sur}(x, y; \theta) \le -\log (K-1).
    \end{align*}
    \item $\L_{ce}$ grows as $\L_{sur}$ grows:
    \begin{align*}
        \L_{ce}(x, y;\theta)\ge -\log(1-(K-1)e^{\L_{sur}(x, y; \theta)}). 
    \end{align*}
\end{enumerate}
\end{prop}


\begin{cor}\label{cor:gue}
    Let $\L_c= \L_{ce}$ and $\L_a=\L_{sur}$. We denote this game as $\G_{ue}$. Then this unlearnable example game $\G_{ue}$ has a Stackelberg equilibrium $(\A_{ue}^*, \theta_{ue}^*)$.
\end{cor}

\paragraph{Adversarial setting.}
Recent studies \cite{tao2021better} have suggested that adversarial training can mitigate the impact of unlearnable example attacks.
Then poisoning approaches \cite{fu2022rem, 2023inf} were proposed to work for adversarially trained models. 

Our unified game framework also includes the adversarial scenario in which the victim is aware of the potentially poisoned training set and decides to deploy adversarial training and the attacker still wants to prevent the victim from obtaining an available model.

We  need only to modify the training loss $\L_c$ to adversarial settings.
Here we consider the two widely-used adversarial losses:
\begin{itemize}
    \item Adversarial Loss \cite{madry2017at}:
    \begin{align*}
       \L_{adv}(x, y;\theta)\coloneqq \max\limits_{||\mu||_\infty\le \epsilon_d} \L_{ce}(x+\mu, y;\theta)
    \end{align*}
where $\epsilon_d$ is the adversarial training radius. 
    \item TRADES Loss \cite{zhang2019trades}:
    \begin{align*}\label{eq:trades}
    \L_{tra}(x, y;\theta) \coloneqq 
    &\frac{1}{\lambda}\max\limits_{||\mu||_\infty\le \epsilon_d}\text{KL}(f_\theta(x)||f_\theta(x+\mu))\\
    &+\L_{ce}(x, y;\theta).
    \end{align*}
\end{itemize}
The following lemma states that $L_{adv}$ and $L_{tra}$ satisfy Assumption \ref{ass:l-lip}.
\begin{lemma}
\label{ln-ad11}
      $\L_{adv}(x, y; \theta)$ and $\L_{tra}(x, y; \theta)$ are continuous. 
\end{lemma}
As a consequence of Theorem \ref{thm:exist}, the equilibrium exists in the adversarial scenario.
\begin{cor}\label{cor:gue-at}
Let $\L_c= \L_{adv}$ or $\L_{tra}$, and $\L_a= \L_{sur}$. We denote this game by $\G_{at}$.
Then this unlearnable example game $\G_{at}$ has a Stackelberg equilibrium $(\A_{at}^*, \theta_{at}^*)$.
\end{cor}
It implies that there exists a poisoning attack which can degrade the classification performance of a model that obtains approximately optimal parameters through adversarial training on the poisoned training set. We mainly consider the case with $\L_c=\L_{tra}$ in the following paper.

\begin{remark}
\label{remark-acc}
Similar to \cite{ATgame},  
by using the loss function $\L_{cw}(f_{\theta}(x), y)=\max_{l\neq y}[f_{\theta}(x)]_l- [f_{\theta}(x)]_y$, 
we can prove that the corresponding game equilibrium exists and gives the most powerful poison attack in that the victim has the lowest test accuracy among all networks within the hypothesis space,  when trained on poisoned training set.
\end{remark}
%

\section{Unlearnable Example Generator} \label{sec:method}
In the preceding section, we theoretically established the unlearnable example game framework and demonstrated the existence of equilibrium in various scenarios.

In this section,  by solving the game using a first-order algorithm, we propose a novel unlearnable example generator that can effectively generalize to future data added to the training set. 
By doing so, we aim to reduce the high dependency of unlearnable examples on the poisoning ratio to better align with the real-world need for protecting unauthorized data.
Moreover, this game-theoretic approach can extend to the adversarial setting.

\subsection{Protection of Unauthorized Data}
In the era of big data, countless personal data is uploaded to servers through various applications, published on the Internet, and can be accessed by others without restrictions.
If these data are obtained without authorization by malicious individuals and used to train machine learning models, there is a risk that the user's personal privacy information being misused.
As a means of protecting unauthorized data, unlearnable examples are used to preprocess upcoming data before its release, rendering models trained on the processed data unusable.

\paragraph{Challenges.} 
However, there are some common issues that hinder the deployment and application of existing unlearnable example methods in the real world. 
Challenges include
\begin{itemize}
    \item Dependency on extremely high poisoning ratio.  
    \item Inefficiency of poison generation.
    \item Vulnerability to recovery approaches.
\end{itemize}
First, even a slight decrease in the poisoning ratio can result in a substantial decrease in unlearnability \cite{huang2021unlearnable}.
When a poisoned training set is diluted by a few new clean samples, the protection of the data will be destroyed.
It brings an essential obstacle to the application of unlearnable examples in the real world.
Furthermore, most of the existing unlearnable example attacks are based on an enormous number of iterative propagation and back-propagation \cite{adv_poison}.
When it comes to redeploying poisoning to new data flow, the issue of inefficiency becomes more prominent. 
Last but not least, adversarial training can restore the usability of models trained on training sets polluted by poisons that have not been specifically designed for adversarial training \cite{tao2021better}. 
Class-wise perturbations are easier to recover compared to sample-wise perturbations \cite{sandoval2022poison-faster}.

\subsection{Generalizable Poison Generator}
To mitigate the aforementioned problems, we propose a novel approach that differs from existing methods. Instead of empirically optimizing perturbations to create shortcuts for learning, we suggest using a game-theoretic framework to train an unlearnable example generator that can generalize to unseen data effectively.

In detail, our approach involves having the attacker in the game utilize a poison generator $g_\omega$ based on an encoder-decoder with parameters $\omega$. 
That is, $\A(x,y)=g_\omega(x)$ which outputs sample-wise perturbations.
The activation for the final layer is $\epsilon\cdot \tanh(\cdot)$ to control the perturbation within the poison radius $\epsilon$, i.e. $||g_w(x)||_\infty \le \epsilon$.

Now the classifier aims at minimizing the payoff function:
\begin{align*}
    \J_c(w,\theta)=\E_{(x,y)\sim S}\big[ \L_c(x+g_w(x), y;\theta)\big]
\end{align*}
by choosing parameters
\begin{align*}
    \theta^*\in \text{BR}_{\eta}(w) = \{\theta|\J_c(w, \theta) < \inf_{\theta'}\J_c(w, \theta') + \eta\}.
\end{align*}
The attacker chooses generator parameters $w^*$ to minimize the payoff function:
\begin{align}
    \J_a(\omega,\theta) \coloneqq 
\sup_{\theta\in \text{BR}_{\eta}(\omega)}  \{-\E_{(x,y)\sim \S} \big[\L_a(x, y; \theta)\big]\}. 
\end{align}
To make it feasible in practical computation, here we consider the loss on clean training set $\S$ instead of the data distribution $\Dd$ in Equation \eqref{equ:payoff-func-attacker}.

Once a generator $g_\omega$ is well trained on a given training set, it costs only forward propagation to generate poisons for images in the training set and potentially for future images that may be added to the training set.
Due to the flexible selection of $\L_c$ and $\L_a$, the generator is able to deal with different scenarios, including standard and adversarial settings.


\subsection{Compute the Game Equilibrium}

In fact, a Stackelburg game can be simplified to a bi-level optimization problem:
\begin{eqnarray}
\label{equ:bi-level}
        &\min_{w, \theta} l(w, \theta) \\ 
    \text{s.t.} \quad \theta &\in \argmin_{\theta'} h(w, \theta'). \nonumber
\end{eqnarray}
%
In the context of an unlearnable example game, we have $l=\J_a$ and $h=\J_c$.
Directly solving a bi-level optimization problem requires the calculation of the second-order Hessian matrices, which makes it impractical for common machine learning tasks with high-dimensional inputs and parameters.

We leverage previously proposed BOME \cite{bome} and dynamic barrier gradient descent algorithm \cite{DynamicBGD} to efficiently train an unlearnable example generator, the detailed algorithm is presented in algorithm \ref{alg:algorithm}.
BOME is a recently developed first-order bi-level optimization algorithm that employs a value function approach to transform the bi-level problem into a single-level constrained optimization problem. 
Then the single-level problem is solved by the dynamic barrier gradient descent algorithm, which has the notable advantage of not requiring the lower-level problem to possess a unique solution. 

\paragraph{BOME.}
Assume $h(w, \cdot)$ can attain a minimum for each $w$.
Problem \eqref{equ:bi-level} is equivalent to the following constrained optimization (even for nonconvex $h$):
\begin{align*}
    &\min_{w, \theta} l(w, \theta) \\
    \text{s.t.} \quad q(w, \theta)&\coloneqq h(w, \theta) - h(w, \theta^*(w))\le 0,
\end{align*}
where $h(w, \theta^*(w))=\min_\theta h(w, \theta)$.
In practice, we approximate $\theta^*(w)$ by $\theta^T(w)$ which is the $T$ step gradient descent of $h(w, \cdot)$ over $\theta$ for some step size $\alpha>0$ :
\begin{equation}
    \theta^{t+1}(w)=\theta^{t}(w) - \alpha \nabla_\theta h(w, \theta^{t}(w)).
    \label{eq:T-step}
\end{equation}
Then we obtain an estimate of $q(w, \theta)$:
\begin{align*}
    \widehat{q}(w, \theta) = h(w, \theta) - h(w, \theta^T(w)).
\end{align*}

\paragraph{Dynamic barrier gradient descent.}
The method is to iteratively update $(w, \theta)$ with step size $\beta >   0$ to reduce $l$ while controlling the decrease of $q$:
\begin{equation*}
    (w_{k+1}, \theta_{k+1}) \leftarrow (w_{k}, \theta_{k}) - \beta (\nabla l(w_{k}, \theta_{k}) + \lambda_{k} \nabla \widehat{q}(w_{k}, \theta_{k}))
\end{equation*}
where 
\begin{align*}
    \lambda_k = \max(\frac{\phi_k - \left \langle \nabla l(w_{k}, \theta_{k}), \nabla\widehat{q}(w_{k}, \theta_{k})\right \rangle}{||\nabla \widehat{q}(w_{k}, \theta_{k})||^2}, 0)
\end{align*} 
and $\phi_k= \rho ||\nabla \widehat{q}(w_{k}, \theta_{k})||^2$ by default with a hyper-parameter $\rho >0$.

\section{Experiments}
\label{sec:experiments}

In this section, we conduct comprehensive experiments to validate the effectiveness of GUE on popular benchmark datasets. Here, we consider not only standard training but also adversarial training.

\subsection{Experiment Settup}
We mainly conduct experiments on image classification datasets: CIFAR-10, CIFAR-100 which have been commonly used in the poisoning literature. We use ResNet-18 \cite{resnet} as classifier $f(\theta)$ and U-Net \cite{unet} as poison generator $g_{w}$ during training.  If not explicitly mentioned, we focus on the reasonable setting with poison radius $\epsilon = 8/255$. 

We use the test accuracy on clean test set to assess the effectiveness of unlearnable example attacks; the lower accuracy implies that the attack is stronger to prevent the model from learning information from the poisoned training dataset. 

\begin{algorithm}[tb]
\caption{Training of unlearnable example generator}
\label{alg:algorithm}
\textbf{Input}: Training set $\S$, inner step $T$, inner and outer step size $\alpha, \beta$, batch size $b$, train epochs $e$.\\
\textbf{Output}: Learned poison Generator $g_w$

\begin{algorithmic}[1] 
\STATE Initialize classifier $f_\theta$ and attacker $g_w$
\FOR{$k = 1$ to $e$}
    \STATE Sample a mini-batch $\{(x_i, y_i)\}_{i=1}^b \sim \S$
    \STATE Compute $\theta^T(w_k)$ by $T$ steps gradient descent on $J_c(w_k, \cdot)$ starting from $\theta_k$(like Eq.\eqref{eq:T-step})
    \STATE set $\widehat{q}(w_k, \theta_k) = J_c(w_k, \theta_k) - J_c(w_k, \theta^T(w_k))$
    \STATE Update $(w, \theta)$:
    $$\theta_{k+1}\leftarrow \theta_{k} -\beta (\nabla J_a(\theta_{k}) + \lambda_{k} \nabla_{\theta_k} \widehat{q}(w_{k}, \theta_{k}))$$
    $$w_{k+1}\leftarrow w_{k}-\beta (\lambda_{k} \nabla_{w_k} \widehat{q}(w_{k}, \theta_{k}))$$
    where $\lambda_k = \max(\frac{\phi_k - \left \langle \nabla J(\theta_{k}), \nabla_{\theta_k}\widehat{q}(w_{k}, \theta_{k})\right \rangle}{||\nabla \hat{q}(w_{k}, \theta_{k})||^2}, 0)$ and $\phi_k= \rho ||\nabla \widehat{q}(w_{k}, \theta_{k})||^2$.
    \\(Set $\rho=1.5$ and $T=10$ as default)
\ENDFOR
\STATE \textbf{return} $g_w$
\end{algorithmic}
\end{algorithm}

\subsection{Standard Training}
In this subsection, we conduct experiments under standard training setting to verify the effectiveness of our GUE attack by directly solving the game $\G_{ue}$ defined in Corollary \ref{cor:gue}.

We train 50 epochs to generate GUE with algorithm \ref{alg:algorithm}, using SGD optimizer with learning rate 0.01 for classifier $f_\theta$ and SGD optimizer with learning rate 0.1 for attacker $g_{w}$. And use Adam \cite{kingma2014adam} with learning rate 0.001 for the inner T steps approximation. And for evaluation of unlearnable examples, we train a model on poisoned dataset for 100 epochs using SGD optimizer with an initial learning rate of 0.01 that is decayed by a factor of 0.1 at the 75-th and 90-th training epochs. The optimizer is set with momentum 0.9 and weight decay $5\times 10^{-4}$.

\paragraph{Compare to different unlearnable example attacks.}
We first evaluate the performance of our GUE and different state-of-the-art poisoning methods under standard training, including DeepConfuse \cite{zhou2019confuse}, Unlearnable Examples(EM) \cite{huang2021unlearnable}, Targeted Adversarial Poisoning(TAP) \cite{adv_poison} and ShortcutGen \cite{shortcutgen}. 
As shown in Table \ref{table_methods}, our GUE can substantially decrease the clean test accuracy
and performs better than all methods except EM.
Furthermore, as we can see from Figure \ref{fig-evaluation}, our GUE remains unlearnably effective throughout all training epochs, while all other attacks exhibit an accuracy peak at the beginning of the training.

\begin{table}[t]
	\centering
	\begin{tabular}[c]{cccc}
    	\toprule
		{Poison method} & {CIFAR10} & {CIFAR100} \\
		\midrule
        None(Clean) &  92.36 & 70.59\\
		EM & 10.16 & 1.90\\ 
		TAP & 20.28 & 15.40\\
  	\midrule
        DeepConfuse  &  22.74 & 25.73\\
        ShortcutGen & 24.42 & 8.62 \\
        GUE & 13.25 & 8.35\\
		\bottomrule
	\end{tabular}
    \caption {Test accuracy of ResNet-18 trained on poisoned data from different unlearnable example attacks on CIFAR.}
    \label{table_methods}
\end{table}

\begin{figure}[t]
\centering
\includegraphics[width=0.9\columnwidth]{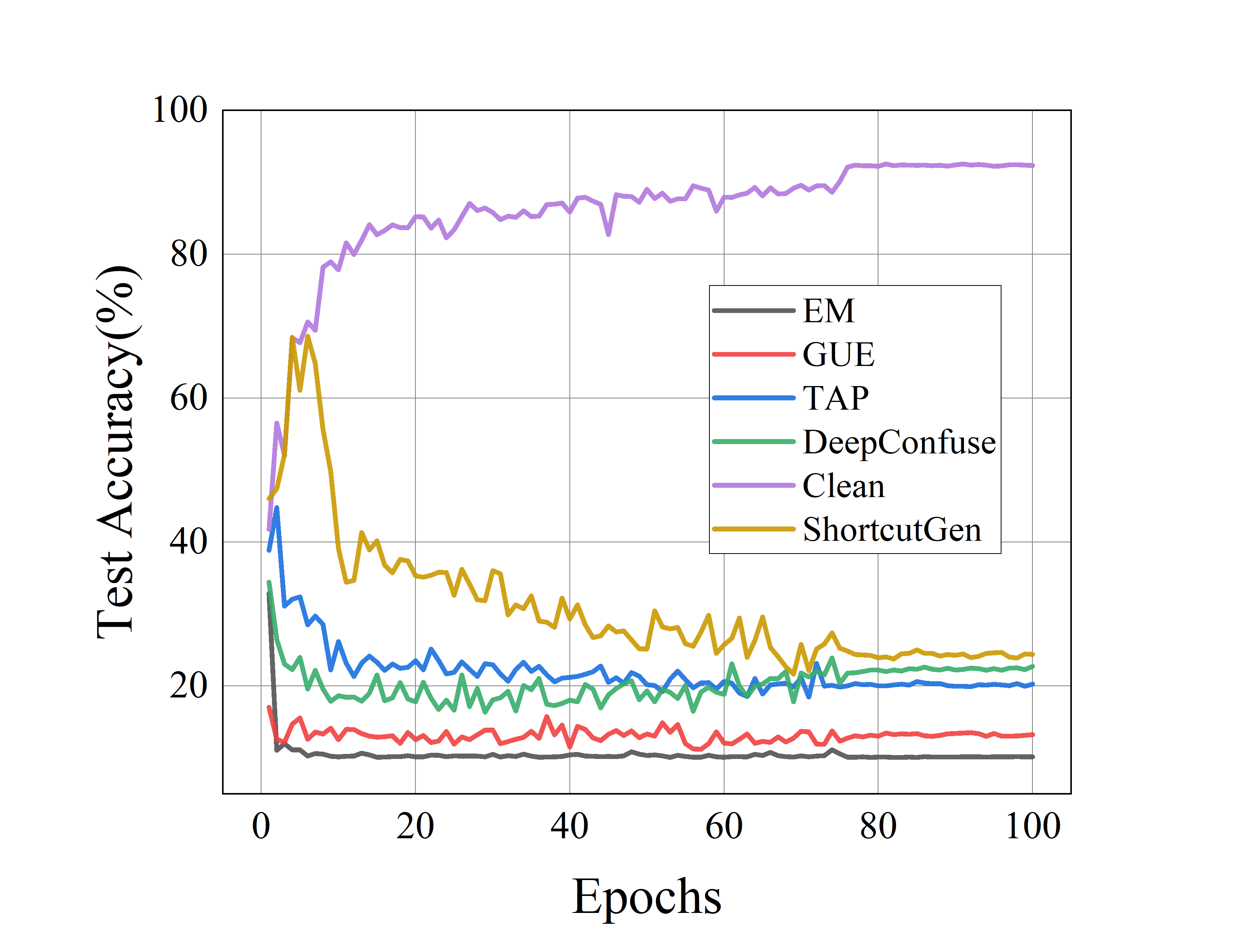} 
\caption{Test accuracy curves of ResNet-18 trained on poisoned data from different unlearnable example attacks on CIFAR-10.}
\label{fig-evaluation}
\end{figure}


\paragraph{Different percentages of data to train $g_w$ and generalizability.} In practice, it is not always possible that the attacker has access to the full training dataset, the victim classifier may collect more data to train a model. 
It has been observed that other unlearnable example attacks such as EM or TAP would fail to poison the model when there is a small proportion of clean data, such as $10\%$ from the training dataset. Thus, when there are more clean data, these attacks require regenerating corresponding unlearnable examples on the entire training data after adding new clean data. 
This motivates us to examine the generalizability of the poison generator $g_w$ trained on only a proportion of randomly selected training examples, and when there are more clean data, we can use $g_w$ to generate unlearnable examples for these data.
\begin{figure}[t]
\centering
\includegraphics[width=0.9\columnwidth]{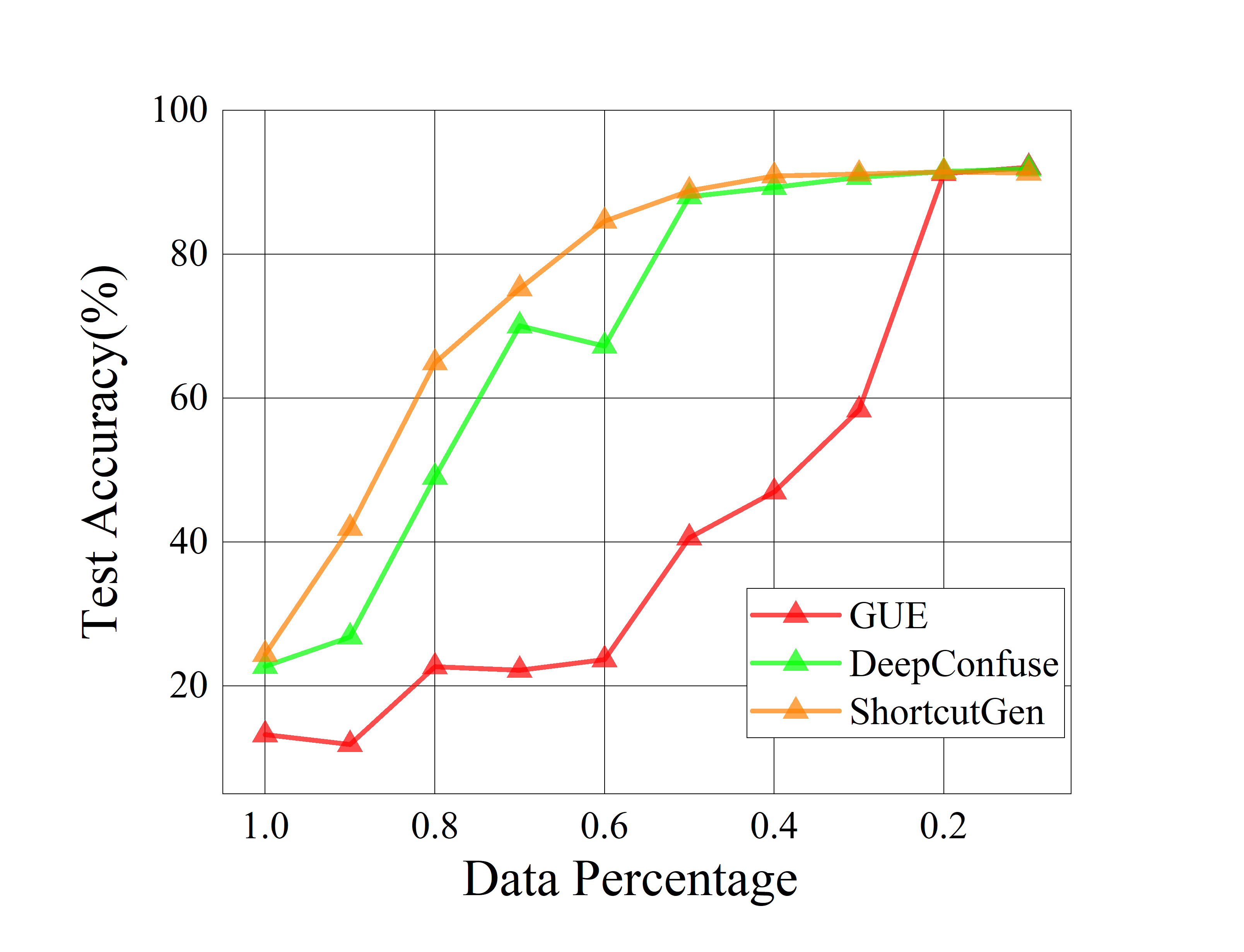} 
\caption{Test accuracy of ResNet-18 trained on poisoned CIFAR-10, where the poison generator is trained on different percentage of training data.}
\label{fig-part1}
\end{figure}

We use different percentages of training data to train the poison generator, then use the corresponding generator to generate unlearnable examples on the entire training dataset. The results are shown in Figures \ref{fig-part1} and \ref{fig-part2}. We can observe that our GUE's generalizability is much better than DeepConfuse and ShortcutGen. Specifically, our GUE is still effective in degrading the test accuracy to about $20\%$ when we only use $60\%$ of the training data in CIFAR10. This indicates that the GUE-trained poison generator can generalize to unseen data well.

\begin{figure}[t]
\centering
\includegraphics[width=0.9\columnwidth]{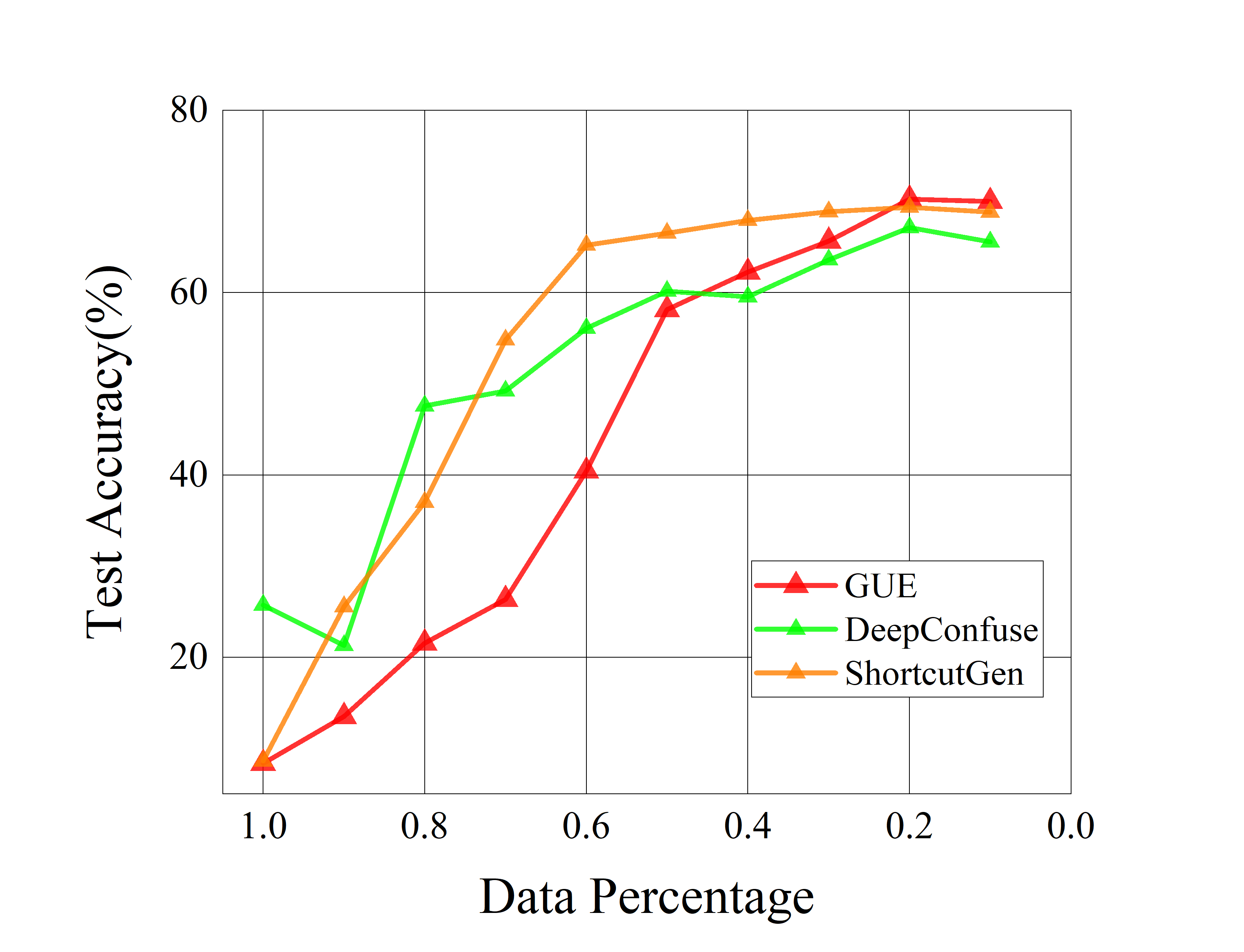} 
\caption{Test accuracy of ResNet-18 trained on poisoned CIFAR-100, where the poison generator is trained on different percentage of train data.}
\label{fig-part2}
\end{figure}


\paragraph{Transferability on different model architectures.} We generate a GUE with fixed classifier structure as ResNet-18.  For this reason, a natural question to ask is whether the poison generator trained with one model architecture is still effective when the victim classifier adopts a different architecture. Table \ref{table_models} shows that the
poisoning effects of our GUE optimized against a ResNet-18 classifier can be transferred to other model architectures. (To save space, we denote DeepConfuse and ShorcutGen as DC and SG.) Specifically, we can observe that the transferability of our GUE is more stable than that of DeepConfuse and ShortcutGen. Our GUE is capable of reducing the test accuracy to $12.97\% \sim 14.68\%$ across various model architectures.
\begin{table}[ht]
	\centering
	\begin{tabular}[c]{ccccc}
    	\toprule
		{Poison method} & {Clean} & {GUE} & {DC} &{SG}\\
		\midrule
		ResNet-18 & 92.36 & 13.25 &22.74 & 24.42\\
	    VGG16 & 91.18 & 13.72 & 25.35 & 12.32\\
        ResNet-50 &  92.14 & 12.97 & 20.56 & 17.35 \\
        DenseNet-121 & 92.10 & 13.71 & 21.44 & 16.59\\
        WRN-28-10 & 92.86 & 14.68 &29.02 & 25.09\\
		\bottomrule
	\end{tabular}
    \caption {Transferability of GUE from ResNet-18 to other model architectures on CIFAR-10.}
    \label{table_models}
\end{table}

\paragraph{Effects of defenses against GUE.}
We have demonstrated the efficacy of our GUE in the setting where the victim classifier only uses standard training. However, the victim classifier could employ several defenses proposed against poison attacks. Thus, we test the effectiveness of several popular defenses against our GUE.

Here, we firstly study whether data augmentations can mitigate our GUE. Following previous work, we test a diverse set of data augmentations, including Mixup \cite{zhang2017mixup}, Cutout \cite{devries2017cutout}, and Cutmix \cite{yun2019cutmix}.
Since adversarial training is widely considered as an effective defense against unlearnable examples attacks, we also test adversarial training with adversarial radius $\epsilon_d = 2/255$ and $4/255$. 

Table \ref{table_defences} shows that all data augmentation methods fail to mitigate our GUE. Adversarial training can counteract the poison effect, but our GUE is still effective when the adversarial radius is small $\epsilon_d=2/255$. It is more effective than other existing methods, such as EM, TAP, both of which achieve test accuracy exceeding $70\%$.

\begin{table}[ht]
	\centering
	\begin{tabular}[c]{cc}
    	\toprule
		{Defenses} & {Clean test accuracy} \\
		\midrule
            Baseline(Clean) &  92.36 \\
		  Mixup & 11.83 \\
            Cutout & 14.71\\
            Cutmix & 19.84 \\
            Adv training($\epsilon_d = 2/255$) & 22.55\\
            Adv training($\epsilon_d = 4/255$) & 76.96 \\
		\bottomrule
	\end{tabular}
    \caption {Evaluating GUE against different defenses.}
    \label{table_defences}
\end{table}

\subsection{Adversarial Training}
We find that our GUE cannot defend against adversarial training for a large adversarial radius in the preceding subsection. 
To solve this problem, we consider solving the game $\G_{at}$ in which the victim classifier also adopts adversarial training. And we denote the attacks as GUE-AT($\epsilon, \epsilon_d$) where $\epsilon$ is the poison radius and $\epsilon_d$ is the adversarial radius in the game $\G_{at}$.

We train 150 epochs to generate GUE-AT with algorithm \ref{alg:algorithm}, the optimizers are set as the standard training.
We set $\lambda=1$ in trade-off loss, the inner maximization (that is, generation of adversarial examples) of the adversarial training is solved by 10-step PGD with a step size of $\epsilon_d/4$.
For adversarial training evaluation experiments, all training settings are the same as the standard training except that the initial learning rate is changed to 0.1.

\paragraph{Different adversarial training perturbation radii.} We train models using different adversarial training perturbation radii on these unlearnable examples. The adversarial training radius $\epsilon_d$ ranges from $0/255$ to $4/255$. It is important to note that when $\epsilon_d=0$, the models are trained using the standard training method. Table \ref{table_gue_at} reports the accuracies of the trained models under different unlearnable example attacks. 

We can see that our GUE-AT outperforms the SOTA method REM over $7.48\% -9.71\%$ when adversarial radius $\epsilon_d=2/255$. For a larger adversarial radius, our GUE-AT is also effective. In total, these experiments show that our method can be applied to different adversarial training perturbation radii while keeping a significant poison effect.

\begin{table}[t]
	\centering
	\begin{tabular}[c]{cccc}
    	\toprule
		{Poison method} &{$\epsilon_d =0$}& {$\epsilon_d = 2/255$}  & {$\epsilon_d = 4/255$} \\
		\midrule  
        None(Clean) & 92.36 & 92.41 &89.82\\
        INF  & 88.78 & 80.91 & 77.96 \\
        \midrule  
		REM(8-2) & 21.14 & 29.75 & 74.20\\
        GUE-AT(8-2) & 16.11 & 22.27 & 60.11 \\
        \midrule  
        REM(8-4) & 26.15 & 30.90 & 44.91 \\
        GUE-AT(8-4) & 17.86 &  21.19 &52.79\\
		\bottomrule
	\end{tabular}
    \caption {Test accuracy of ResNet-18 adversarially trained with   radius $\epsilon_d = 0, 2/255, 4/255$. 
    INF is from \cite{2023inf}.}
    \label{table_gue_at}
\end{table}

\subsection{Ablation Studies}
In this subsection, we conduct experiments to provide an empirical understanding of the proposed surrogate loss. 
As discussed  when introducing the surrogate loss,
the cross-entropy loss has no upper bound and does not provide a criterion for convergence.
Therefore, we propose the surrogate loss $\L_{sur}$, which is upper bounded and equivalent to $\L_{ce}$. Empirically, we compute the game equilibrium with $L_a = \L_{ce}$ and $\L_a = \L_{sur}$, and observe the gradient of $\J_{a}$ and the poisoned loss $\J_c(w, \theta)$ in each batch of data during training. 

Figure \ref{fig:gradient} (a) shows that $||\nabla_{\theta} \J_a||_2$ with $\L_{ce}$  is approximately two orders of magnitude larger than with $\L_{sur}$, and diverges. A straightforward method to overcome the explosion of gradients is gradient clipping. Hence, we use gradient clipping to control the gradient norm $||\nabla_{\theta} \J_a||_2 = 10$.
However, in Figure \ref{fig:gradient}(b) it can be seen that, as a necessary condition for convergence, poisoned loss $\J_c(w, \theta)$ cannot converge when we use $\L_{ce}$ or $\L_{ce}$ with gradient clipping. Only our proposed loss has guarantee of convergence. This verifies the necessity of our proposed loss $\L_{sur}$ to compute the equilibrium.

\begin{figure}[t]
\centering
    \centering
    \includegraphics[width=0.8\columnwidth]{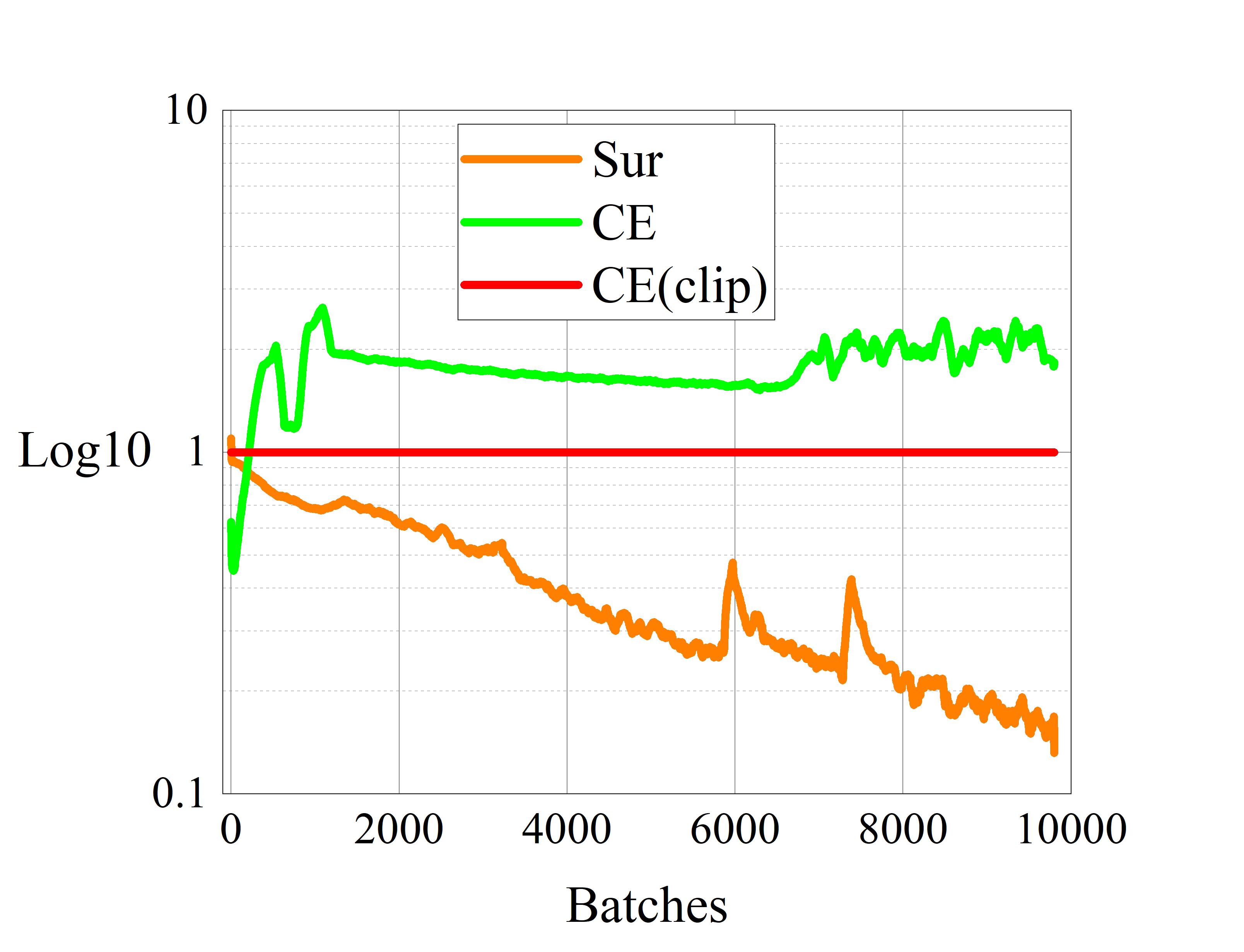}
\centerline{(a)  $||\nabla_{\theta} \J_a||_2$}
    \centering
    \includegraphics[width=0.8\columnwidth]{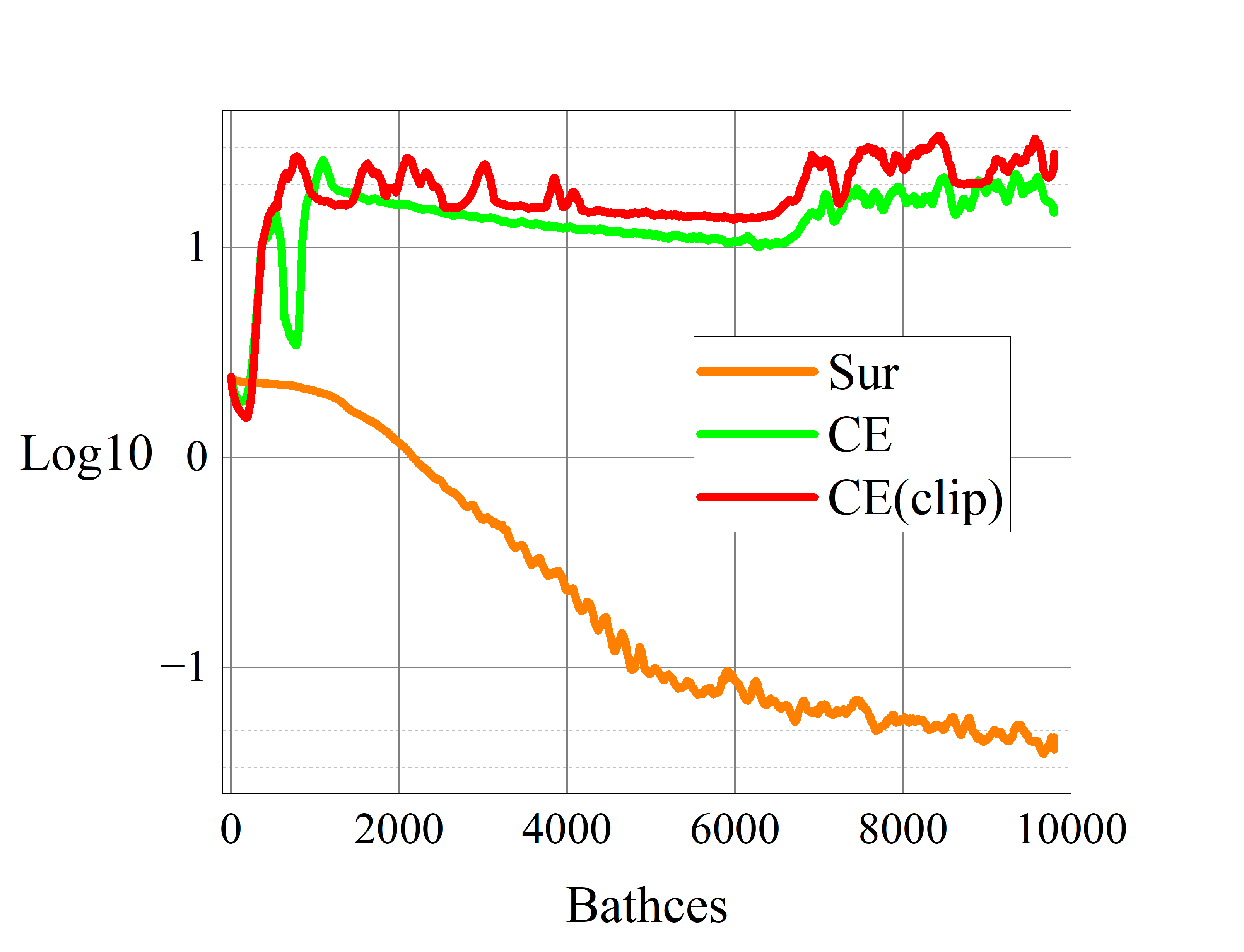}\\
    \centering{(b) Poisoned loss $\J_c(w, \theta)$}
\caption{The training curves of $||\nabla_{\theta} \J_a||_2$ and  $\J_c(w, \theta)$ when we use different loss function in $\J_{a}$ to compute the equilibrium: $\L_{ce}, \L_{sur}$ and $\L_{ce}$ with gradient clipping. We take $\log_{10}$ of all values for better visualization.
}
\label{fig:gradient}
\end{figure}

\section{Proofs}
In this section, we give  proofs of the theoretical results in the paper
and give a complete statement of Remark \ref{remark-acc}.

\subsection{Proof of Theorem \ref{thm:exist}}
\label{pf-th3}
We first introduce some lemmas about the semicontinuity of set-valued maps and then give the complete proof of the existence of the Stackelberg equilibrium of unlearnable examples game.

Similar to Lemma 3.2 in \cite{ATgame}, we can prove the following properties of payoff functions.
\begin{lemma}
\label{lemma:cont}
     The payoff functions $J_c(\A, \theta)$ and $J_a(\theta)$ are continuous.
\end{lemma}

\begin{lemma}\label{inf-cont}
Let $X, Y$ be two Hausdorff spaces and $X$ compact, 
$F:X\times Y\rightarrow \mathbb{R}$ continuous. 
Then $\inf\limits_{x\in X}F(x, y)$ is continuous on $Y$.
\begin{proof}
    It is equivalent to prove that $\forall \epsilon>0, \exists \delta>0$, for any $||y_1-y_2||\le\delta$, we have $|\inf_{x\in X} F(x, y_1) - \inf_{x\in X} F(x, y_2)|\le\epsilon$.
    Since $F(x, y)$ is continuous and $X$ is compact, there exists an $x^*$ for any $y$ such that $F(x^*, y) = \inf_{x\in X} F(x, y)$, so there exist $x_1^*, x_2^*$ correlated with $y_1, y_2$, respectively.

Since $F(x, y)$ is continuous, $\forall \epsilon>0, \exists \delta(x)>0$, such that when $||y_1-y_2||\le\delta(x)$, we have $|F(x, y_1) - F(x, y_2)|\le\epsilon$.
Hence $\forall \epsilon>0$, there exits a $\delta(x_2^*)$, such that
    \[\inf_{x\in X} F(x, y_1) - \inf_{x\in X} F(x, y_2)\le F(x_2^*, y_1) - F(x_2^*, y_2)\le\epsilon\]
    and there exits a $\delta(x_1^*)$ such that
    \begin{equation*}
        \begin{aligned}
            & \inf_{x\in X} F(x, y_1) - \inf_{x\in X} F(x, y_2) \\
             & \ge F(x_1^*, y_1) - F(x_1^*, y_2) \ge  -\epsilon.
        \end{aligned}
    \end{equation*}
Let $\delta = \min\{\delta(x_1^*), \delta(x_2^*)\}$. Then $\forall \epsilon>0, \exists \delta>0$, such that for any $||y_1-y_2||\le\delta$, $|\inf_{x\in X} F(x, y_1) - \inf_{x\in X} F(x, y_2)|\le\epsilon$ holds, that is, $\inf\limits_{x\in X}F(x, y)$ is continuous on $Y$.
\end{proof}
\end{lemma}

\begin{lemma}\label{lm:set-map}
    \cite{Aubin1984AppliedNA} 
Let $X, Y$ be two Hausdorff spaces, $G$ a set-valued map from $Y$ to $X$, and $W$ a real-valued function defined on $X$. Suppose that $W$ is lower semicontinuous, $G$ is lower semicontinuous on $Y$. Then the marginal function $V(y)=\sup\limits_{x\in G(y)}W(x)$ is lower semicontinuous on $Y$.
\end{lemma}

\begin{lemma}\label{set-cont}
    Let $X, Y$ be two compact subsets in a metric space, $F:X\times Y\rightarrow \mathbb{R}$ a continuous function, 
    and $G(y)= \{x\in X:F(x, y)<\inf \limits_{z\in X}F(z, y) + \eta\}$ where $\eta>0$ is a constant. Then $G(y):Y\rightarrow X$ is lower semicontinuous.
\end{lemma}   
\begin{proof}
The fact that $G(y)$ is lower semicontinuous is equivalent to proving for any $y\in Y$, any sequence $\{y^m\}_{m=1}^\infty$ convergent to $y$, and $\forall x^0\in G(y)$, there exists a sequence $\{x^m\}_{m=1}^\infty$ converging to $x^0$ such that $x^m\in G(y^m)$ holds for sufficiently large $m$.

    Since $x^0\in G(y)$, we have $F(x^0, y)<\inf \limits_{x}F(x, y)+\eta$, and $X$ is compact, so there exists a sequence $\{x^m\}_{m=1}^\infty$ that converges to $x^0$. Because $x^m\rightarrow x^0$ and $y^m \rightarrow y$, and $F(x, y)$ are continuous, we have $\lim\limits_{m\rightarrow\infty}F(x^m, y^m)=F(x^0, y)$. According to the Lemma \ref{inf-cont}, since $X$ is compact, $\inf\limits_{x\in X} F(x, y)$ is continuous on $Y$. Then
    \begin{align*}
         \lim\limits_{m\rightarrow\infty}F(x^m, y^m) &=F(x^0, y)<\inf \limits_{x}F(x, y)+\eta \\&= \lim\limits_{m\rightarrow\infty}\inf\limits_{x\in X} F(x, y^m) +\eta.
    \end{align*}
Then for a sufficiently large $m$, we have $F(x^m, y^m)<\inf\limits_{x\in X} F(x, y^m) +\eta$, i.e. $x^m\in G(y^m)$ for a sufficiently large $m$. Therefore, $G(y)$ is lower semi-continuous on $y$.
\end{proof}

Now we can prove the main theorem:
\begin{theorem}
    (Restate of Theorem \ref{thm:exist}) The unlearnable example game $\mathcal{G}$ has a Stackelberg equilibrium.
\end{theorem}
\begin{proof}
By Lemma \ref{lemma:cont}, we have $J_c(\A, \theta )$ is continuous, and by Assumption \ref{ass:comp}, $\Theta$ and $\Gamma$ are compact. Then, for any $\A$, $R_{\eta}(\A)$ is nonempty. Let $J(\theta) \coloneqq - \E_{(x,y)\sim \Dd} \big[\L_a(x, y; \theta)\big]$. 
Then the existence of a Stackelberg equilibrium is equivalent to $\exists \A^*\in \Gamma$ such that:
\[\sup_{\theta\in R_{\eta}(\A^*)}J(\theta)= \inf_{\A} \sup_{\theta\in R_{\eta}(\A)}J(\theta). \]
By Lemma \ref{set-cont}, we have that $R_{\eta}(\cdot)$ is lower semicontinuous. 
And since $J(\theta)$ is lower semi-continuous on $\Theta$, by Lemma \ref{lm:set-map}, $V(\A)= \sup\limits_{\theta\in R_{\eta}(\A)} J(\theta)$ is lower semicontinuous. Then $V(\A)$ can reach the minimum as $\Gamma$ is compact. Then $\exists \A^*$ such that
\[V(\A^*)= \inf_{\A\in \Gamma} V(\A)= \sup_{\theta\in R_{\eta}(\A^*)}J(\theta). \]
Since $R_{\eta}(\A)$ is a nonempty subset,  there exists a $\theta^*\in R_{\eta}(\A^*)$.
Thus, we have proved that there exist $\A^*\in \Gamma$ and $\theta^*\in R_{\eta}(\A^*)$ such that $(\A^*, \theta^*)$ is a Stacklberg equilibrium.
\end{proof}

\subsection{Proof of Lemma \ref{ln-ad11}}
\begin{lemma} (Restatement of Lemma \ref{ln-ad11})
      The adversarial loss function $ L_{adv}(x, y; f_\theta)=\max\limits_{||\mu||_\infty \le \epsilon_d} L_{ce}(f_\theta(x+\mu), y)$ is continuous.
  \begin{proof}
We show that $L_{adv}(x, y; f_\theta)$ is continuous at $\theta$, and the continuity for $x$ is similar. It is clear that $L_{ce}(f_\theta(x+\mu), y)$ is continuous on $\theta$. Therefore, for any $\gamma>0$, there exists a $\Delta(x) > 0$ (depending on $x$), such that for any $||\theta_1 - \theta_2||\le \Delta(x)$ we have $|L_{ce}(x, \theta_1) -L_{ce}(x, \theta_2)|\le \gamma$.

Let $L_{adv}(x_1, \theta_1)= \max\limits_{||\mu||_\infty\le \epsilon_d} L_{ce}(f_{\theta_1}(x_1+\mu), y)$ and $L_{adv}(x_2, \theta_2)= \max\limits_{||\mu||_\infty\le \epsilon_d} L_{ce}(f_{\theta_2}(x_2+\mu), y)$. Then for $||\theta_1 - \theta_2||\le \Delta(x_2)$:
        \begin{equation*}
            \begin{aligned}
                & L_{ce}(x_1, \theta_1) - L_{ce}(x_2, \theta_2)  \\ &\ge L_{ce}(x_2, \theta_1) - L_{ce}(x_2, \theta_2)\ge -\gamma
            \end{aligned}
        \end{equation*}
        and for $||\theta_1 - \theta_2||\le \Delta(x_1)$:
       \begin{equation*}
            \begin{aligned}
                & L_{ce}(x_1, \theta_1) - L_{ce}(x_2, \theta_2)  \\ &\le L_{ce}(x_1, \theta_1) - L_{ce}(x_1, \theta_2)\le \gamma.
            \end{aligned}
        \end{equation*}
        Thus, for any $\gamma>0$, there exists a constant $\Delta=\min\{\Delta(x_1), \Delta(x_2)\} > 0$, such that for any $||\theta_1 - \theta_2||\le \Delta$ we have $|\max\limits_{x'}L_{ce}(x', \theta_1) - \max\limits_{x'}L_{ce}(x', \theta_2)|\le \gamma$, which means $L_{adv}(x, y; f_\theta)$ is continuous at $\theta$.
    \end{proof}
\end{lemma}

Similarly, we can prove that the trade-off adversarial loss is continuous:
\begin{lemma}
       The TRADES Loss function $ L_{tra}(x, y; f_\theta)= L_{ce}(f_\theta(x), y) + \frac{1}{\lambda}\max\limits_{||\mu||_\infty\le \epsilon_d}KL(f_\theta(x)||f_\theta(x+\mu))$ is continuous.
\end{lemma}

\subsection{Discrete Loss for Lowest Accuracy}
As we stated in Remark \ref{remark-acc}, in order to exactly degrade the victim classifier's test accuracy, we consider a discrete loss function similar to \cite{ATgame}:
\[\L_{acc}(f_{\theta}(x), y) =
\begin{cases}
0  & \L_{cw}(f_{\theta}(x), y)\ge0 \\
-1 & \L_{cw}(f_{\theta}(x), y)<0\\
\end{cases}\]
where $\L_{cw}(f_{\theta}(x), y)=\max_{l\neq y}[f_{\theta}(x)]_l- [f_{\theta}(x)]_y$ is the Carlini-Wagner loss \cite{carlini2017towards}. Then the negative value of population risk: $-\E_{(x,y)\sim \Dd} \big[\L_{acc}(x, y; \theta)\big]$
is exactly the classification accuracy.

Though here $\L_{acc}$ is not continuous, we also can extend our existence Theorem \ref{thm:exist} to this case:
\begin{cor}
    Let $L_a = L_{acc}$, $L_c = L_{ce}$, we denote this Stackelberg game as $\mathcal{G}_{acc}$. Then the unlearnable example game $\mathcal{G}_{acc}$ has a Stackelberg equilibrium $(\A_{acc}^*, \theta_{acc}^*)$ .
    \begin{proof}
        Since $L_{cw}(f_{\theta}(x), y)$ is continuous at $\theta$, it is easy to verify that $-L_{acc}(f_{\theta}(x), y)$ is lower semi-continuous on $\theta$. 
        Furthermore $J(\theta)$ is lower semi-continuous on $\theta$. And the remaining proof is similar to before.
    \end{proof}
\end{cor}

\section{Conclusion}
While unlearnable example attacks have been well studied in various formulations, a unified theoretical framework remains unexplored. In this paper, we formulate the unlearnable example attack as a Stackelberg game that encompasses a range of scenarios. We propose a novel attack GUE by directly computing the Stackelberg equilibrium using a first-order optimization method and by using a better loss function. 
Our approach possesses a stronger theoretical intuition, whereas other methods are primarily based on empirical practices. 
Furthermore, we show that the game equilibrium gives the most powerful poison attack in the sense that the victim neural network has the lowest test accuracy among all networks within the same hypothesis space.

Extensive experiments demonstrate the effectiveness of GUE in different scenarios, in particular, the generalizability of poison generator trained with low percentage of training data and its effectiveness against adversarial training. 

\paragraph{Limitations} Our algorithm cannot converge well when the adversarial radius is larger, hence not effective to poison adversarial trained model with more aggressive budgets. Future research is needed to find more effective algorithms to solve the unlearnable example games.

\section{Acknowledgements}
This work is partially supported by the NSFC grant No.11688101 and NKRDP grant No.2018YFA0306702.
The authors thank anonymous referees for their valuable comments.
%

\end{document}